\documentclass[12pt]{article}
\usepackage[utf8]{inputenc}

\RequirePackage{kpfonts}
\RequirePackage[sf,bf,small,raggedright]{titlesec}

\usepackage{pgf,tikz}
\usetikzlibrary{arrows}

\usepackage{hyperref}
\usepackage{graphicx} 
\usepackage{booktabs} 

\usepackage{amsfonts} 
\usepackage{amsmath}
\usepackage{amssymb}
\usepackage{tcolorbox}
\usepackage{bbm}

\usepackage{color}

\usepackage{float}

\usepackage{comment}

\definecolor{csim}{rgb}{0.3,0.7,0}

\usepackage[longnamesfirst,authoryear,sort&compress]{natbib}
\bibpunct[, ]{[}{]}{,}{n}{}{,}%

\newcommand{\R}{\mathbb{R}}

\newcommand{\PP}{\mathbb{P}}

\newcommand{\wt}{\widetilde}

\newcommand{\ol}{\overline}

\newcommand{\field}[1]{\mathbb{#1}}
\newcommand{\E}{\field{E}}
\newcommand{\EXP}{\E}
\newcommand{\PROB}{\field{P}}

\def\IND{\mathbbm{1}}

\usepackage{fancyhdr}
\newtheorem{theorem}{Theorem}
\newtheorem{lemma}[theorem]{Lemma}

\newtheorem{corollary}[theorem]{Corollary}

\newtheorem{proposition}[theorem]{Proposition}

\newcommand{\qed}{\hfill $\Box$}
\newenvironment{proof}{\par\noindent{\bf Proof:}}{\qed \par\medskip\noindent}

\title{On the quality of randomized approximations of Tukey's depth
    \thanks{Simon Briend acknowledges the support of Région Ile de France.
G\'abor Lugosi acknowledges the support of Ayudas Fundación BBVA a
Proyectos de Investigación Científica 2021 and
the Spanish Ministry of Economy and Competitiveness, Grant
PGC2018-101643-B-I00 and FEDER, EU.}
\author{
  Simon Briend \\
  Unidistance Suisse\\
  3900, Brig, Suisse 
\and  
G\'abor Lugosi \\
Department of Economics and Business, \\
Pompeu  Fabra University, Barcelona, Spain \\
ICREA, Pg. Lluís Companys 23, 08010 Barcelona, Spain \\
Barcelona School of Economics
\and
Roberto Imbuzeiro Oliveira \\
IMPA, Rio de Janeiro, RJ, Brazil
}
}

\date{}

\begin{document}

\maketitle

\begin{abstract}
Tukey's depth (or halfspace depth) is a widely used measure of centrality for multivariate data. However, exact computation of Tukey's depth is known to be a hard problem in high dimensions. As a remedy, randomized approximations of Tukey's depth have been proposed. In this paper we explore when such randomized algorithms return a good approximation of Tukey's depth. We study the case when the data are sampled from a log-concave isotropic distribution. We prove that, if one requires that the algorithm runs in polynomial time in the dimension, the randomized algorithm correctly approximates the maximal depth $1/2$ and depths close to zero. On the other hand, for any point of intermediate depth, any good approximation requires exponential complexity.
\end{abstract}

\section{Introduction}
\label{sec:intro}

Ever since Tukey introduced a notion of data depth \cite{10029477185},
it has been an important tool of data analysts to measure centrality of data points
in multivariate data. Apart from Tukey's depth (also called halfspace depth), many other depth measures
have been developed, such as simplical depth \cite{liu1988notion,liu1990notion},
projection depth \cite{Liu92,ZuSe00},
a notion of ``outlyingness'' \cite{Sta81,Don82}, and the zonoid depth \cite{DyMoKo96,koshevoy1997zonoid}.
Each of these notions offer distinct stability and computability properties that make them suitable for different applications \cite{mosler2021choosing}.  For surveys of depth measures and their  applications we refer the reader to \cite{Mos02,Alo06,dyckerhoff2016exact,rousseeuw1999depth,NaScWe19}. 

Tukey's depth is defined as follows: for $x\in \R^d$ and unit vector $u\in S^{d-1}$ (where $S^{d-1}$ is the unit sphere of $\R^d$ under the euclidean norm), introduce the closed halfspace
\[
  H(x,u) = \left\{ y \in \R^d \ : \ \langle y,u\rangle \leq \langle x,u\rangle \right\}~,
\]
where $\langle \cdot,\cdot\rangle$ is the usual scalar product on $\R^d$. Given a set of $n$ data points $\{x_1,\ldots,x_n\}$ in $\R^d$, for each $x\in \R^d$, define the directional depth
\[
  r_n(x,u) = \frac{1}{n}\sum_{i=1}^n \IND_{x_i\in H(x,u)}~.
\]
For any $x\in\R^d$, its depth  in the point set $\{x_1,\ldots,x_n\}$ is defined as
\[
  d_n(x) = \inf_{u\in S^{d-1}} r_n(x,u)~.
\]
Note that, due to the normalization in our definition, $d_n(x)\in [0,1]$ for all $x\in \R^d$. A point that maximizes Tukey's depth is called a Tukey median.
Tukey's depth possesses properties expected of a depth measure. It is affine invariant, it vanishes at infinity, and it is monotone decreasing on rays emanating from the deepest point. It is also robust under a symmetry assumption  \cite{DoGa92}.

A well-known disadvantage of Tukey's depth is that even its
approximate computation is known to be a {\sc np}-hard problem \cite{AmKa95,BrChIaLaMo08,JOHNSON197893}, presenting challenges for applications. While fast algorithms exist for computing the depth of the deepest point in two dimensions \cite{chan2004optimal}, the computational complexity grows exponentially with the dimension. In \cite{chan2004optimal}, a maximum-depth computation algorithm of complexity $\mathcal{O}(n^{d-1})$ is given.

The curse of dimensionality affects several other depth measures,
posing significant challenges in multivariate analysis. To address
these challenges, focus has been put on developing approximation
algorithms. In \cite{MR4207999}, the importance of finding such algorithms is presented and {\sc mcmc} methods are proposed in \cite{SHAO2022114278} for approximating the projection depth. In \cite{Zuo19newapproach}, an approximate version of Tukey's depth is introduced and an algorithm with linear time complexity in the dimension is provided, though the proposed version may be a poor approximation
of Tukey's depth. 

A natural way of approximating Tukey's depth, proposed independently
by Dyckerhoff \cite{Dyc04} and Cuesta-Albertos and Nieto-Reyes \cite{cuesta2008random},
is a randomized version in which the infimum over all possible directions $u\in S^{d-1}$ in the definition of $d_n(x)$ is replaced by the minimum over a number of randomly chosen directions.
More precisely, let $U_1,\ldots,U_k$ be independent identically distributed vectors sampled uniformly on the unit sphere $S^{d-1}$, and define the \emph{random Tukey depth} (with respect to the point set $\{x_1,\ldots,x_n\}$) as
\[
  D_{n,k}(x)  = \min_{i=1,\ldots,k} r_n(x,U_i)~.
\]
It is easy to see that for every $x\in\R^d$, $\lim_{k\to\infty} D_{n,k}(x) = d_n(x)$ with probability $1$. However, this randomized approach is only useful if the number of random directions
$k$ is reasonably small so that computation is feasible. The purpose of this paper is to explore the
tradeoff between computational complexity and accuracy. In particular, we may ask how large $k$ has to be
in order to guarantee that, for given accuracy and confidence parameters $\epsilon \in (0,1/2)$ and $\delta\in (0,1)$,  
$|D_{n,k}(x)-d_n(x)|\le \epsilon$ with probability at least $1-\delta$.

It is easy to see that the value of $k$ required to satisfy the property
above may be arbitrarily large. To see this, fix $n\geq 4$ even and
consider the two-dimensional example in which, for $i=[1,n]$, the
points $x_i=(x_{i,1},x_{i,2})$
are defined by 
\[
   x_{i,1}= \frac{i}{n}~, \qquad x_{i,2}= a\left(\frac{i}{n}\right)^2
\]
where $a>0$ is a parameter. We want to evaluate the depth of $x_{n/2}$ in the point set $S=\{x_1,\cdots,x_{n/2-1},x_{n/2+1},\cdots,x_n \}$.  Since $x_{n/2}$ lies outside of the convex envelope of $S$, its depth with regard to the set $S$ is $0$. However, to "see" this depth we need to evaluate the depth in a direction $u$ such that
$$ \langle x_{n/2-1},u \rangle> \langle x_{n/2},u \rangle~,$$
and
$$\langle x_{n/2+1},u \rangle > \langle x_{n/2},u \rangle~.$$
An illustration, presented in Figure \ref{fig:TukeyCurve}, shows that by choosing $a$ arbitrarily small, the set where $u$ must be sampled to detect the depth of $x_{n/2}$ becomes arbitrarily small, and consequently $k$ must be chosen arbitrarily large to accurately estimate the Tukey depth.

\begin{figure}
\includegraphics[width=8cm]{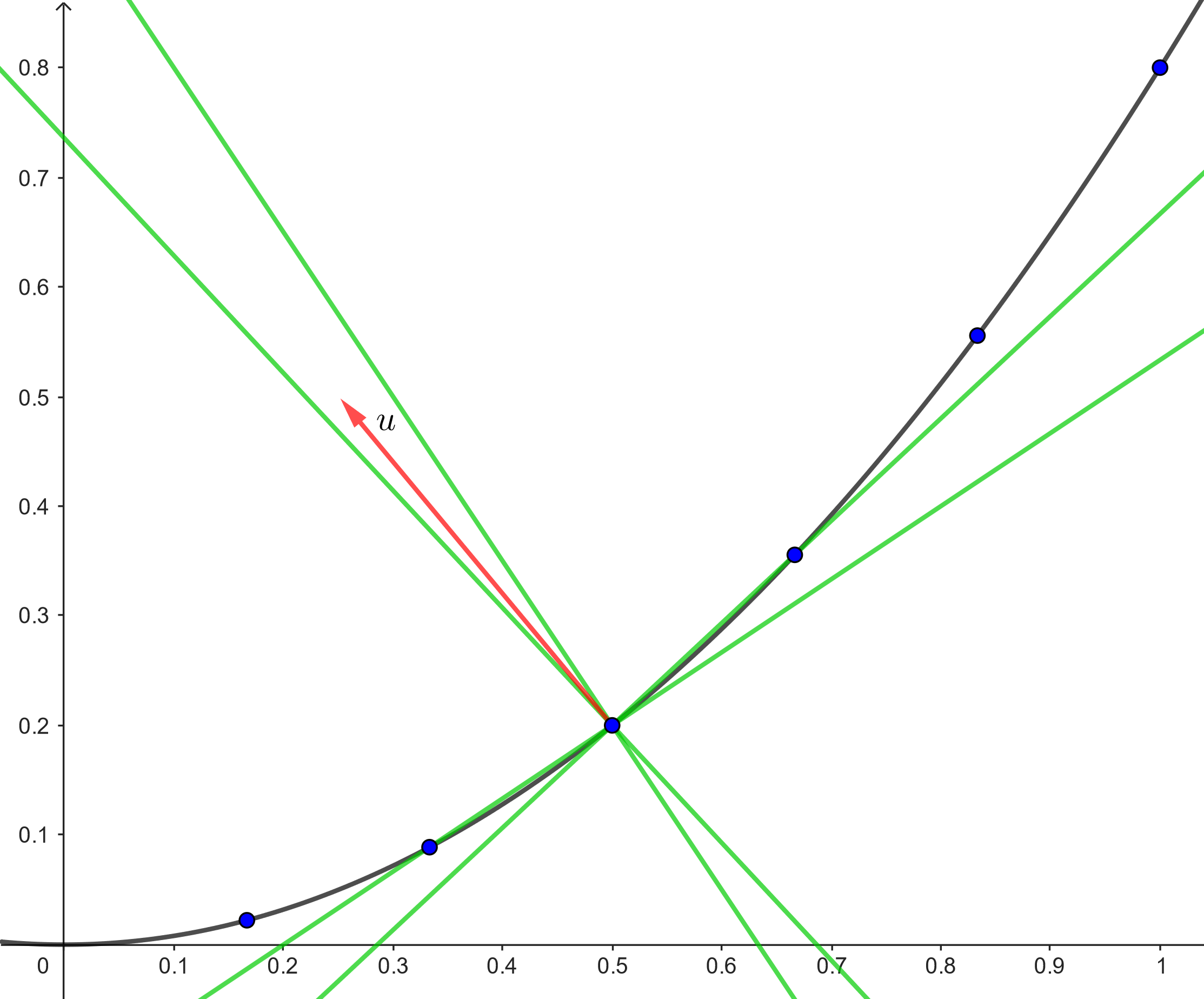}
\centering
\caption{Illustration of the dataset $x_i=\left(i/6,0.8(i/6^2\right)$ for $i\in[6]$. In green, the lines between which vector $u$ must lie to detect that the depth of $x_3$ is $0$.}
\label{fig:TukeyCurve}
\end{figure}

In order to exclude the anomalous behavior of the example above,
we assume that the points $x_i$ are drawn randomly from an isotropic
log-concave distribution $\mu$. Recall that a distribution $\mu$ is log-concave
if it is absolutely continuous with respect to the Lebesgue measure, with density
$f$ of the form $f(x)=e^{-g(x)}$ where $g:\R^d\to \R$ is a convex function.
$\mu$ is isotropic if for a random vector $X$ distributed by $\mu$,
the covariance matrix $\EXP (X-\EXP X)(X-\EXP X)^T$ is the identity matrix. Examples of log-concave
distributions include Gaussian distributions and the uniform distribution on a convex body in $\R^d$.

For random data, one may introduce the ``population'' counterpart of $r_n$ defined by
\[
   \ol{r}(x,u) = \mu(H(x,u))~.
\]
Similarly, the population versions of the Tukey depth and randomized Tukey depth
are defined by
\[
   \ol{d}(x)=  \inf_{u\in S^{d-1}}\ol{r}(x,u)
\quad \mbox{and} \quad
   \ol{D}_k(x)= \min_{i=1,\ldots,k} \ol{r}(x,U_i)~.
 \]
 Remark that if $\mu$ has a density, then Tukey's depth lies between $0$ and $1/2$. As it was observed in \cite{cuesta2008random}  and \cite{ChGaRe18}, as long as $n\gg d$,
 the population versions of the Tukey depth $\ol{d}(x)$ and randomized Tukey depth $\ol{D}_k(x)$
 are good approximations of $d_n(x)$ and $D_{n,k}(x)$, respectively.
This follows from standard uniform convergence results of empirical process theory
based on the {\sc vc} dimension.
The next lemma quantifies this closeness. For completeness we include its proof in the Appendix.

\begin{lemma}
\label{TheEmpGap}
Let $\delta>0$. 
If $X_1,\ldots,X_n$ are independent, identically distributed random
vectors in $\R^d$, then
\[
\PROB\left\{   \sup_{x\in \R^d} |\ol{d}(x) - d_n(x)|\ge c\sqrt{\frac{d\log(n)}{n}} + \sqrt{\frac{\log(1/\delta)}{2n}} \right\} \le \delta
\]
where $c$ is a universal constant.
Also, given any fixed values of $U_1,\ldots,U_k$, 
\[
\PROB\left\{   \left. \sup_{x\in \R^d} |\ol{D}_k(x) - D_{n,k}(x)| \geq c\sqrt{\frac{\log(n)\min(d,\log(k))}{n}} + \sqrt{\frac{\log(1/\delta)}{2n}} \right| U_1,\ldots,U_k \right\} \leq \delta.
\]
\end{lemma}

Thanks to Lemma \ref{TheEmpGap}, in the rest of the paper we restrict our attention to
the population quantities $\ol{d}(x)$ and $\ol{D}_k(x)$ and we may forget the data points $X_1,\ldots,X_n$.
In particular, we are interested in finding out for what points $x\in \R^d$ and $k\ge 0$ the random Tukey depth $\ol{D}_k(x)$ is a good
approximation of $\ol{d}(x)$. To this end, we fix an accuracy $\epsilon >0$ and a confidence
level $\delta>0$ and ask that
\begin{equation}
\label{eq:goal}
\ol{D}_k(x) - \ol{d}(x)\le \epsilon \quad
\text{holds with probability at least $1-\delta$}.
 \end{equation}
(Note that, by definition, $\ol{D}_k(x) \ge \ol{d}(x)$ for all $x$ and $k$.)
The main results of the paper show an interesting trichotomy: for most ``shallow'' points (i.e., those
with $\ol{d}(x) \le \epsilon$),  we have $\ol{D}_k(x)\le \epsilon$ with probability at least $1-\delta$
even for $k$ of \emph{constant} order, depending only on $\epsilon$ and $\delta$.
When $x$ has near maximal depth in the sense that $\ol{d}(x) \approx 1/2$ (note that such points may
not exist unless the density of $\mu$ is symmetric around $0$), then for values of $k$ that are slightly
larger than a linear function of $d$, \eqref{eq:goal} holds. However, in sharp contrast with this,
for points $x$ of intermediate depth, $k$ needs to be exponentially large in $d$ in order to guarantee
\eqref{eq:goal}. Hence, roughly speaking, the depth of very shallow and very deep points can be
efficiently approximated by the random Tukey depth but for all other points, any reasonable
approximation by the random Tukey depth requires exponential complexity.

\subsection{Related literature}

In \cite{cuesta2008random}, various properties of the random Tukey depth are investigated and  good experimental behavior is reported.
The maximum discrepancy between $d_n$ and its randomized approximation has also been studied in \cite{MR4165498}. They establish conditions under which $\sup_{x\in \R^d} \left(\ol{D}_k(x)- \ol{d}(x)\right) \to 0$
as $k\to \infty$ and provide bounds for the rate of convergence. As opposed to the global view
presented in \cite{MR4165498}, our aim is to identify the points $x$ for which the random Tukey depth approximates
well $\ol{d}(x)$ for values of $k$ that are polynomial in the dimension. 

In \cite{Brazitikos22halfspacedepth}, it is shown that the average depth $\int _{\R^d}\ol{d}(x) d\mu(x)$ is exponentially
small in the dimension when $\mu$ is log-concave. In \cite{ilienko2025strong}, the set of ``deep points'' for a mesure and its weighted empirical counterpart are compared. On a similar line of questions, \cite{brunel2019concentration} studies convergence of the empirical level sets 
when the data points are drawn independently from the same distribution, and \cite{ChMoWa13} studies the quality of other randomized approximations of the Tukey depth
for point sets in general position.

 
\subsection{Contributions and outline}

As mentioned above, the main results of this paper show that, for isotropic log-concave
distributions, the quality of approximation of the random Tukey depth varies dramatically,
depending on the depth of the point $x$. 

\subsubsection*{Most points have a small random Tukey depth}
In Section \ref{sec:shallow} we establish results related to shallow points. 
It follows from \cite[Theorem 1.1]{Brazitikos22halfspacedepth}
(together with the solution of Bourgain's slicing problem by \cite{KlLe24})
and Markov's inequality
that all but an exponentially small fraction of points are shallow in the sense that, for all $\epsilon>0$,
\[
   \mu\left( \left\{x\in \R^d: \ol{d}(x) > \epsilon \right\} \right) \le \frac{e^{-cd}}{\epsilon}~,
 \]
 where $c>0$ is a universal constant. The main result of Section \ref{sec:shallow} is that, in high dimensions,
 not only most points are shallow but most points even have a small random Tukey depth for $k$
 of \emph{constant} order, only depending on the desired accuracy. In particular, Theorem \ref{thm:shallow}  leads to the following.

\begin{tcolorbox}
\begin{corollary}
\label{cor:shallow}
Assume that $\mu$ is an isotropic log-concave measure on $\R^d$.
There exist universal constants $c,\kappa,C>0$ such that for any $\epsilon,\delta,\gamma >0$, if
 \[
   k = \left\lceil
     \max\left(C,\frac{4}{\epsilon}\log\frac{3}{\gamma},\frac{2}{c}\log\frac{4}{\delta} \right)  \right\rceil~, 
   \]
   and the dimension $d$ is so large that
   \[
     d \ge \max\left(\left(\frac{3(k+1)}{\gamma}\right)^{1/\kappa}, \frac{64 k \log(1/\epsilon)^2}{\pi}\log\frac{3k}{\gamma},
       \left(\frac{1}{c}\log\frac{6k}{\delta}\right)^2, \left(\frac{2}{\epsilon}\right)^{\kappa} \right)~,
   \]
   then, with probability at least $1-\delta$,
\[
    \mu\left(\left\{x\in \R^d: \ol{D}_k(x) > \epsilon\right\} \right) < \gamma~.
\]
\end{corollary}
\end{tcolorbox}

Of course, $\ol{D}_k(x) \le \epsilon$ implies that $\ol{d}(x) \le \epsilon$ and, in particular,
that $\ol{D}_k(x) -\ol{d}(x) \le \epsilon$. Thus, Corollary \ref{cor:shallow} implies that the random Tukey depth 
of \emph{most} points (in terms of the measure $\mu$) is a good approximation of the Tukey depth
after taking just a constant number of random directions. All of these points are shallow in the sense
that $\ol{d}(x) \le \epsilon$. 

It is natural to ask whether the Tukey depth of every shallow point is well approximated by 
its random version. However, this is false as the following example shows.

\medskip
\noindent
{\bf Example.}
Let $\mu$ be the uniform distribution on $[-\sqrt{3},\sqrt{3}]^d$ so that $\mu$ is
isotropic and log-concave on $\R^d$. If $x=(\sqrt{3},0,\ldots,0)$, then $\ol{d}(x)=0$, but
it is a simple exercise to show that $\ol{D}_k(x) \ge 1/4$ with high probability, unless $k$
is exponentially large in $d$.

\subsubsection*{Intermediate depth is hard to approximate}

Arguably the most interesting points are those whose depth is in the intermediate range,
bounded away from $0$ and $1/2$. Unfortunately, for all such points, the random Tukey depth
is an inefficient approximation of the Tukey depth.
In Section \ref{sec:intermediate} we show that
for all points in this range, the random Tukey depth $\ol{D}_k(x)$ is close to $1/2$, with high probability,
unless $k$ is exponentially large in the dimension. Hence, in high dimensions,
$\ol{D}_k(x)$ fails to efficiently approximate the true depth $\ol{d}(x)$.
In particular, Theorem \ref{thm:intermediate} leads to the following.

\begin{tcolorbox}
\begin{corollary}
\label{cor:intermediate}
Assume that $\mu$ is an isotropic log-concave measure on $\R^d$ and let $\delta \in (0,1)$.
For any $\gamma \in (0,1/2)$, there exists a positive constant $c=c(\gamma)$ such that
if $x\in\R^d$ is such that $\ol{d}(x)=\gamma$, then
for every $\epsilon < c$, if $k \le (\delta/2) e^{d\epsilon^2 /2}$, then, with probability
at least $1-\delta$,
\[
     \ol{D}_k(x) - \ol{d}(x) \ge \epsilon~.
\]
\end{corollary}
\end{tcolorbox}

\subsubsection*{Points of maximum depth are easy to localize}

As mentioned above, if $\mu$ has no atoms,  the Tukey depth $\ol{d}(x)$ of any $x\in \R^d$ is at most $1/2$.
If $\ol{d}(x)=1/2$, then for every $u\in S^{d-1}$, the median of the projection $\langle X, u \rangle$
equals $\langle x, u \rangle$ (where the random vector $X$ is distributed as $\mu$).
Such points are quite special and may not exist at all. If there exists an $x\in \R^d$ with $\ol{d}(x)=1/2$,
then the measure $\mu$ is called \emph{halfspace symmetric} (see
\cite{ZuSe00a,NaScWe19}).
It is easy to see that if $\mu$ is halfspace symmetric and has no atoms, there is a
unique $m\in \R^d$ with $\ol{d}(m)=1/2$. Then $m$ is the unique \emph{Tukey
  median} of $\mu$.
Centrally symmetric measures are halfspace symmetric though the converse does not hold in general.
Remarkably, if $\mu$ is the uniform distribution over a convex body and it is halfspace symmetric,
then it is also centrally symmetric, see  Funk \cite{Fun15},
Schneider \cite{Sch70},  and 
Rousseeuw and Struyf \cite{MR2057920}, who independently proved a more
general result, stating that an atomless distribution is halfspace symmetric if and only if it is angularly symmetric.
See \cite{NaScWe19} for further discussion.

We note that for any log-concave measure, $1/e \le \sup_{x\in \R^d} \ol{d}(x) \le 1/2$, see
\cite[Theorem 3]{NaScWe19}.

If $m\in \R^d$ is such that $\ol{d}(m)=1/2$, then clearly $\ol{D}_k(m)=1/2$ for all $k\ge 1$.
In Section \ref{sec:deep} we show that, for values of $k$ that are only polynomial in $d$, points
with $\ol{D}_k(x) \approx 1/2$ must be close to $m$. Hence, selecting points of high random Tukey depth
efficiently estimates the Tukey median for halfspace symmetric isotropic log-concave distributions.
More precisely, Theorem \ref{thm:deep}, combined with Lemma \ref{TheEmpGap} leads to the following.

\begin{tcolorbox}
\begin{corollary}
\label{cor:deep}
Assume that $\mu$ is an isotropic log-concave, halfspace symmetric measure on $\R^d$ and let $m$ be its unique Tukey median.
Let $X_1,\ldots,X_n$ be independent random vectors distributed as
$\mu$. Let $m_{n,k}\in \R^d$ be an empirical random Tukey median, that is, $m_{n,k}$
is such that $D_{n,k}(m_{n,k}) = \max_{x\in R^d} D_{n,k}(x)$.
There exist
universal constants $c,C>0$ such that for any $\delta \in (0,1)$ and
$\gamma \in (0,c)$,
if $n \ge C d/\gamma^2$ and
\[
  k \geq
  c\left(d\log d+\log(1/\delta)\right)~,
\]
then $\|m_{n,k}-m\| \le C\gamma\sqrt{d}$  with probability at least $1-\delta$.
\end{corollary}
\end{tcolorbox}

By taking $\gamma$ of the order of $1/\sqrt{d}$,
the corollary above shows that, as long as $n\gg d^2$, it suffices to take $O(d\log d)$ random
directions so that the empirical random Tukey median is within
distance of constant order of the Tukey median.
Note that, due to the ``thin-shell'' property of log-concave measures (see,
e.g., \cite{ElLe14}), the measure $\mu$ is concentrated around a
sphere of radius $\sqrt{d}$ centered at the Tukey median $m$ and hence
localizing $m$ to within a constant distance is a nontrivial
estimate.

One may even take $\gamma$ to be smaller order than $1/\sqrt{d}$
and get a better precision with the same value of $k$. However, for
better precision, one requires the sample size $n$ to be larger.

At first glance, the statements of Theorems \ref{thm:intermediate}
and \ref{thm:deep} might seem contradictory. Indeed, the former states
that the random Tukey depth of points of intermediate depth
concentrates around $1/2$ while the latter states that, in the halfspace symmetric
case, it is possible to localize the Tukey median using the points of
high random Tukey depth. Theorem \ref{thm:intermediate}
states that the random Tukey depth is larger than $1/2-c$ for a
(potentially small) constant $c$. In the meantime, meaningful
information is extracted from Theorem \ref{thm:deep} when points of
depth at least $1/2-1/\sqrt{d}$ are considered. This suggests that
even though random Tukey depth of a point of intermediate depth concentrates around $1/2$, it is still
below $1/2-1/\sqrt{d}$ and this allows localization of the Tukey median.

One may wonder whether the Tukey median can be localized by
  maximizing the random Tukey depth for general (not only halfspace
  symmetric) log-concave distributions. While we cannot rule out this
  property, its proof would require new ideas as by Theorem \ref{thm:intermediate},
  even if the Tukey median has depth smaller than $1/2$,
  intermediate points have a random Tukey depth concentrating around
  $1/2$.

\section{Random Tukey depth of typical points}
\label{sec:shallow}

In this section we show that for isotropic log-concave distributions, in high dimensions,
a constant number $k$ of random directions suffice to make the random Tukey depth $\ol{D}_k$
small for most points.
In other words, the curse of dimensionality is avoided in a
strong sense.
In particular, we prove the following theorem that implies Corollary \ref{cor:shallow}
in a straightforward manner.

  \begin{theorem}
    \label{thm:shallow}
Assume that $\mu$ is an isotropic log-concave measure on $\R^d$.
There exist universal constants $c,\kappa>0$ such that the following holds.
Let $\epsilon>0$ and suppose that $d$ is so large that $d^{-\kappa} \le \epsilon/2$.
Then for every $k\le cd^{\kappa}$,
\[
\mu\left(\{x\in \R^d: \ol{D}_k(x) > \epsilon\}\right) \le (1-\epsilon/4)^k + (k+1)d^{-\kappa} + ke^{\frac{-d\pi}{64 k \log(1/\epsilon)^2}}
\]
with probability at least $1-ke^{-ck} -(k+1)e^{-c\sqrt{d}}$ over the choice of directions $U_1, \dots, U_k$.
\end{theorem}

Our main tool is the following extension of Klartag's celebrated 
central limit theorem for convex bodies (\cite{Kla07a}).
Let $G_{d,k}$ denote the Grassmannian of all $k$-dimensional subspaces of $\R^d$
and let $\sigma_{d,k}$ be the unique rotationally invariant probability measure on $G_{d,k}$.

\begin{proposition}
\label{klartag}
(Klartag, \cite{Kla07b}.)
Let the random vector $X$ take values in $\R^d$ and assume that $X$ has an isotropic
log-concave distribution.
Let $S_k$ be a random $k$-dimensional subspace of $\R^d$ drawn from the distribution
$\sigma_{d,k}$. There exist universal 
constants $c,\kappa>0$ such that the following holds: if $k\le cd^{\kappa}$, then
with probability at least $1-e^{-c\sqrt{d}}$, for every measurable set $A\subset S_k$,
\[
   \left|\PROB\{\pi_k(X) \in A \} - \PROB\{N \in A\}\right| \le d^{-\kappa}
\]
where $N$ is a $k$-dimensional normal vector in $S_k$ with zero mean and identity covariance matrix, and $\pi_k$ is the orthogonal projection on $S_k$.
\end{proposition}

\medskip
\noindent
{\bf Proof of Theorem \ref{thm:shallow}:}
First note that the random subspace of $\R^d$ spanned by the independent uniform vectors
$U_1,\ldots,U_k$ has a rotation-invariant distribution and therefore it is distributed by $\sigma_{d,k}$
over the Grassmannian $G_{d,k}$.

For any $u\in S^{d-1}$, define $q(\epsilon,u)$ as the 
$\epsilon$-quantile of the distribution of $\langle X,u\rangle$, that is,
\[
  \mu(\{x: \langle x,u\rangle \le q(\epsilon,u) \}) = \epsilon~.
\]
Observe that, by Proposition \ref{klartag} (applied with $k=1$) and the union bound, 
with probability at least $1-ke^{-c\sqrt{d}}$, 
\[
\mbox{for all $i=1,\ldots,k$,} \quad
q(\epsilon,U_i) \ge \Phi^{-1}(\epsilon/2)
\]
whenever $d$ is so large that $d^{-\kappa}\le \epsilon/2$,
where $\Phi(z)= \int_{-\infty}^z (2\pi)^{-1/2} e^{-t^2/2}dt$ denotes the 
standard Gaussian cumulative distribution
function.

Then, with probability at least $1-ke^{-c\sqrt{d}}$, 
\begin{eqnarray*}
\mu\left(\left\{x:\ol{D}_k(x) > \epsilon \right\}\right)
& = & 
\mu\left(\left\{x: \min_{i=1,\ldots,k} \mu(H(x,U_i)) > \epsilon\right\}\right) \\
& = & 
\mu\left(\left\{x: \langle x,U_i\rangle  > q(\epsilon,U_i) \ \mbox{for all $i=1,\ldots,k$}
\right\} \right)
\\
& \le & 
\mu\left(\left\{x: \langle x,U_i \rangle > \Phi^{-1}(\epsilon/2) \ \mbox{for all $i=1,\ldots,k$}\right\} \right).
\end{eqnarray*}

If the $U_i$ were orthogonal, we could now use Proposition \ref{klartag}.
This is not the case but almost. 
In order to handle this issue, we perform Gram-Schmidt orthogonalization
defined, recursively, by $V_1=U_1$
and, for $i=2,\ldots,k$, 
\[
   R_i = \sum_{j=1}^{i-1} \langle U_i,V_j \rangle V_j  \quad \mbox{and} \quad
    V_i = \frac{U_i-R_i}{\|U_i-R_i\|}~.
\]
Then $V_1,\ldots,V_k$ are orthonormal vectors, spanning the same 
subspace as $U_1,\ldots,U_k$.

Now, we may write
\begin{eqnarray}
\label{eq1}
\lefteqn{
\mu\left(\left\{x:\ol{D}_k(x) > \epsilon\right\}\right)
} 
\nonumber \\*
& \le & 
\mu\left(\left\{x: \langle x,U_i\rangle > \Phi^{-1}(\epsilon/2) \ \mbox{for all $i=1,\ldots,k$} \right\}\right)  \nonumber \\
& \le & 
\mu\left(\left\{x: \langle x,V_i\rangle > \Phi^{-1}(\epsilon/4) \ \mbox{for all $i=1,\ldots,k$}\right\}\right)   \nonumber \\*
& & +
\mu\left(\left\{x: \langle x,U_i-V_i\rangle  > \Phi^{-1}(\epsilon/2)-\Phi^{-1}(\epsilon/4) \ 
\mbox{for some $i=1,\ldots,k$} \right\}\right)  \nonumber \\
& \le & 
\mu\left(\left\{x: \langle x,V_i\rangle  > \Phi^{-1}(\epsilon/4) \ \mbox{for all $i=1,\ldots,k$} \right\}\right)   \nonumber \\*
& & +
\sum_{i=2}^k \mu\left(\left\{x: \langle x,U_i-V_i\rangle  >  \frac{\sqrt{2\pi}}{4\log(1/\epsilon)}
\right\}\right)~,
\end{eqnarray}
where the last inequality follows from the union bound and  Lemma \ref{lem:PhiControl} below.

As $\langle x,V_1\rangle,\ldots,\langle x,V_k\rangle$ are coordinates of the orthogonal projection
of $x$ on the random subspace spanned by $U_1,\ldots,U_k$, we
may use Proposition \ref{klartag} to bound the first term
on the right-hand side of \eqref{eq1}. Let $N_1,\ldots,N_k$ be independent
standard normal random variables. Then by Proposition \ref{klartag},
with probability at least $1-e^{-c\sqrt{d}}$,
\begin{eqnarray*}
\lefteqn{
\mu\left(\left\{x: \langle x,V_i\rangle > \Phi^{-1}(\epsilon/4) \ \mbox{for all $i=1,\ldots,k$}\right\}\right)   
} \\
& \le &
\PROB\{ N_i > \Phi^{-1}(\epsilon/4) \ \mbox{for all $i=1,\ldots,k$} \}   
+ d^{-\kappa}  \\
& = & 
\PROB\{ N_1 > \Phi^{-1}(\epsilon/4)\}^k   
+ d^{-\kappa} \\
& = & 
(1-\epsilon/4)^k
+ d^{-\kappa} 
\end{eqnarray*}
It remains to bound the second term on the right-hand side of (\ref{eq1}).
Once again, we use Proposition \ref{klartag}. By rotational invariance, for $i\in[2,k]$,
the distribution of $ (U_i-V_i)/\|U_i-V_i\|$ is uniform on $S^{d-1}$ and
therefore the distribution of 

$$\mu\left(\left\{x: \langle x,U_i-V_i\rangle > \frac{\sqrt{2\pi}}{4\log(1/\epsilon)}
\right\}\right)$$
is the same as that of
\[
\mu\left(\left\{x: \langle x,W\rangle > \frac{\sqrt{2\pi}}{4\log(1/\epsilon)\|U_i-V_i\|}
\right\}\right)
\]
(if $\epsilon \le 1/2$)
where $W$ is uniformly distributed on $S^{d-1}$. By Lemma \ref{orth} below, with probability at least $1-ke^{-ck}$,

$$\max_{i=1,\ldots,k} \|U_i-V_i\| \le \sqrt{4k/d}.$$
Combining this with Proposition \ref{klartag}, we have that, with probability
at least $1-ke^{-ck} - ke^{-c\sqrt{d}}$,

\begin{align*}
\sum_{i=2}^k \mu\left(\left\{x: \langle x,U_i-V_i\rangle > \frac{\sqrt{2\pi}}{4\log(1/\epsilon)} \right\}\right) & \le kd^{-\kappa} + k\PROB\left\{N > \frac{\sqrt{2\pi}}{4\log(1/\epsilon)}\sqrt{\frac{d}{4k}} \right\} \\
& \le kd^{-\kappa} + ke^{\frac{-d\pi}{64 k \log(1/\epsilon)^2}}~.
\end{align*}

In the proof of Theorem \ref{thm:shallow}, we have used the following two lemmas. 
The proof of the first lemma may be found in the Appendix.

\begin{lemma}\label{lem:PhiControl}
For the standard Gaussian cumulative distribution $\Phi$ 
function and $0<\epsilon<e^{-2}$,
\[
\Phi^{-1}(\epsilon/2)-\Phi^{-1}(\epsilon/4)\geq  \frac{\sqrt{2\pi}}{4\log(1/\epsilon)} ~.
\]
\end{lemma}

\begin{lemma}
\label{orth}
For every $i=1,\ldots,k$, with probability at least $1-e^{-ck}$,
\[
   \|U_i-V_i\| \le \sqrt{\frac{4k}{d}}
\]
where $c$ is a universal constant, possibly different than the previous constant $c$.
\end{lemma}

\begin{proof}
Note that, since $\|R_i\|^2= \langle U_i,R_i\rangle$, 
\[
   \langle U_i,V_i\rangle = \frac{1-\langle U_i,R_i\rangle }{\|U_i-R_i\|} = \sqrt{1-\|R_i\|^2}
   \ge 1-\|R_i\|^2
\]
and therefore
\[
   \|U_i-V_i\|^2= 2(1-\langle U_i,V_i\rangle ) \le 2\|R_i\|^2 = 2 \sum_{j=1}^{i-1} \langle U_i,V_j\rangle^2~.
\]
We may write $U_i=Z_i/\|Z_i\|$ where $Z_i$ is a Gaussian vector in 
$\R^d$ with zero mean and identity covariance matrix. Since $Z_i$
is independent of $V_1,\ldots,V_{i-1}$ and the $V_j$ are orthonormal,
$\sum_{j=1}^{i-1} \langle Z_i,V_j\rangle^2$ is a $\chi^2$ random variable with $i-1$ degrees
of freedom. Thus, $\|U_i-V_i\|^2=2A/B$ is the doubled ratio of a $\chi^2(i-1)$ random variable $A$ and a $\chi^2(d)$
random variable $B$ (which are not independent). 
Then, with probability at least $1-e^{-ck}$,
\[
\|U_i-V_i\|^2 \le \frac{4k}{d}~.
\]
To see this, we may use standard tail bounds
for the $\chi^2$ distribution (see, e.g., \cite{BoLuMa13}). Indeed,
\[
    \PROB\left\{ \frac{A}{B} \ge \frac{4k}{d} \right\} 
\le  \PROB\left\{ A \ge 2k \right\} + \PROB\left\{ B \le d/2 \right\}~.
\]
Using $d\ge k$, both terms may be bounded by $e^{-ck}$ using the inequality of
Remark 2.11 in \cite{BoLuMa13}.

\end{proof}

\section{Estimating intermediate depth is costly}
\label{sec:intermediate}

In this section we prove that, even though the random Tukey depth
is small for most points $x\in \R^d$ (according to the measure $\mu$), whenever the
depth $\ol{d}(x)$ of a point is not small, its random Tukey depth $\ol{D}_k(x)$ is close to
$1/2$, unless $k$ is exponentially large in $d$. This implies that for points whose
depth is bounded away from $1/2$, the random Tukey depth is a poor approximation of $\ol{d}(x)$.

The main result of the section is the following theorem that immediately implies Corollary \ref{cor:intermediate} stated in Section \ref{sec:intro}.

\begin{theorem}
\label{thm:intermediate}
Assume that $\mu$ is an isotropic log-concave measure on $\R^d$ and let
$0<\gamma<1/2$. Let $x\in\R^d$ be such that $\ol{d}(x)=\gamma$ and let $\epsilon >0$.
Then
\[
  \PP\left\{ \ol{D}_k(x)\leq \frac{1}{2}-C_{\gamma}\epsilon\log^2\left( \frac{1}{\epsilon}\right) \right\}  \leq  2ke^{-(d-1)\epsilon^2/2}~,
\]
where $C_{\gamma}>0$ is a constant depending only on $\gamma$.
\end{theorem}

\begin{proof}
  Without loss of generality, we may assume that the origin has maximal depth, that is,
  $\ol{d}(0)=\sup_{x\in \R^d}\ol{d}(x)$. Fix $x\in \R^d$ with $\ol{d}(x)= \gamma$, and note that $\ol{d}(0)\geq \gamma$.
  
The main tool of this proof is L\'evy's isoperimetric inequality (\cite{Sch48}, \cite{Lev51}, see also \cite{ledoux2001concentration}). It states that if the random vector $U$ is uniformly distributed on the sphere
  $S^{d-1}$ and $A$ a is Borel-measurable set such that $\PROB\{U\in
  A\} \geq 1/2$, then for any $\epsilon>0$,
\begin{equation}\label{eq:LevyTheorem}
\PROB\left\{ \inf_{v\in A} \|U-v\| \geq \epsilon \right\} \leq 2e^{-(d-1)\epsilon^2/2}~.
\end{equation}
L\'evy's inequality may be used to prove concentration inequalities for smooth functions of
the random vector $U$. Our goal is to prove that the measure $\mu(H(x,U))$ of the random halfspace
$H(x,U)$ is concentrated around its median $1/2$.

In order to prove smoothness of the function $\mu(H(x,u))$ (as a function of $u\in S^{d-1}$), 
fix $u,v\in S^{d-1}$, $u\neq v$. Consider the $2$-dimensional cone spanned by the segments $(x,u)$ and $(x,v)$
defined by
\[
  C(x,u,v)=\left\{x +  a u+b v: \ a,b \in \R^+  \right\}~.
\]


Denote by $\mathcal{H}$ the only two-dimensional affine space
containing $x, \ x+u, \ x+v$.

We also define $P_{\mathcal{H}}$ as the orthogonal projection onto $\mathcal{H}$.
Denoting by $\widetilde{\mu}=P_{\mathcal{H}}\# \mu $ the pushforward of $\mu$ by $P_{\mathcal{H}}$, and $\widetilde{H}(x,u)=P_{\mathcal{H}} \left( H(x,u) \right)$, we have
\[
  \mu(H(x,u))= \widetilde{\mu}\left( \widetilde{H}(x,u)\right)~.
\]
Thus, after projecting on the plane $\mathcal{H}$, it suffices to control
\begin{eqnarray}
\label{eq:conemeasure}
  |\mu(H(x,u)) - \mu(H(x,v))| & = &
  \left|\widetilde{\mu}\left(
                                    \widetilde{H}(x,u)\right)-\widetilde{\mu}\left(
                                    \widetilde{H}(x,v)\right)\right|
 \nonumber  \\ & = & \left|\widetilde{\mu}
                     \left(C(x,u^{\perp},v^{\perp})\right) - \widetilde{\mu}\left(C(x,-u^{\perp},-v^{\perp})\right) \right|
 \nonumber  \\ & \le & \widetilde{\mu}
                     \left(C(x,u^{\perp},v^{\perp})\right) + \widetilde{\mu}\left(C(x,-u^{\perp},-v^{\perp})\right) ~,
\end{eqnarray}
that is, the measure of two cones in a $2$ dimensional
affine space. Here, given an arbitrary orientation to the plane
$\mathcal{H}$,  $u^{\perp}$ and $v^{\perp}$ are the only unit
vectors orthogonal to $u$ and $v$, respectively, in $\mathcal{H}$ such
that $u^{\perp}$ and $v^{\perp}$ are rotated $90$ degrees counter-clockwise from $u$ and $v$,
see Figure \ref{fig:cones}.

\begin{figure}[h]
\includegraphics[width=8cm]{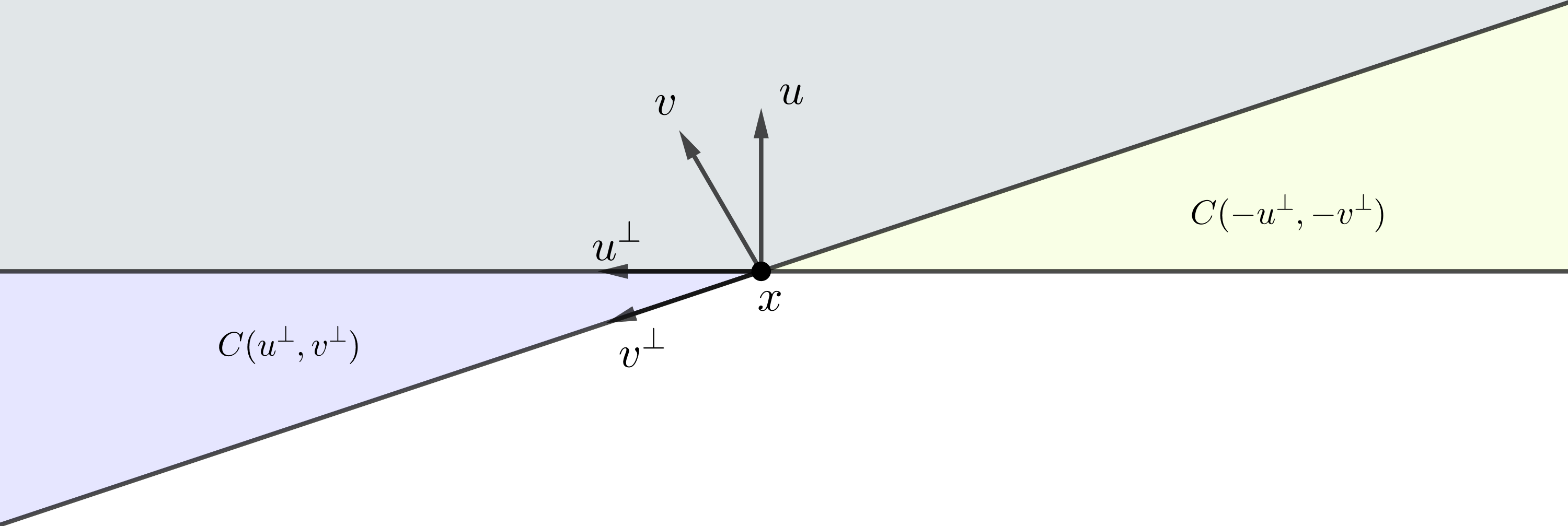}
\centering
\caption{Illustration of the cones $C(x,u^{\perp},v^{\perp})$ and $C(x,-u^{\perp},-v^{\perp})$.}
\label{fig:cones}
\end{figure}

Since the measure $\wt\mu$ is itself an isotropic log-concave measure
(see \cite[Section 3]{saumard2014log}, \cite{Pre73}), the problem
becomes two dimensional.
Next, we show 
that neither $\|x\|$ nor $|m_v|$ are too
large, where $m_v$ denotes the median of the random variable $\langle
X,v\rangle$. (Note that $m_v$ is uniquely defined since $\langle
X,v\rangle$ is log-concave and therefore has a unimodal density.)

In the Appendix we gather some useful facts on log-concave densities.
In particular,  Lemma \ref{exponentialupperbound} shows that any one dimensional
 log-concave density with unit variance is upper bounded by an
 exponential
 function centered at the median of the log-concave density. Since
 $\ol{d}(0)\leq \ol{r}(0,v)$ for all $v\in S^{d-1}$, Lemma \ref{exponentialupperbound} used with both $v$ and $-v$
implies that there exist universal constants $c_1,c_2>0$ such that
\[
  \ol{d}(0) \leq c_1e^{-c_2 |m_v| }~.
\]
Since $\ol{d}(0)\geq\gamma$, we have
\begin{equation}\label{medianlocalisation}
 c_2|m_v|\leq \log(c_1/\gamma)~.
\end{equation}
Moreover, since $\ol{d}(x)=\gamma$, the same argument leads to 
\[
  \gamma \leq c_1 e^{-c_2|\langle x,v\rangle -m_v|}~.
\]  
For $x\neq 0$, we use the above with $v=x/\left\|x\right\|$ and the inequality $|a-b|\geq |a|-|b|$ to yield

$$ c_2\|x\| \leq \log(c_1/\gamma)+c_2 | m_{x/\|x\|}|~, $$
which, put together with \eqref{medianlocalisation}, implies
\begin{equation}\label{eq:xlocalisation}
\left\| x \right\| \leq c\log(c_1/\gamma)~,
\end{equation}
for a positive constant $c$. In particular, $\left\|
  P_{\mathcal{H}}(x)\right\|\leq c\log(c_1/\gamma)$.
We use this inequality to control the measure of half
spaces around $x$.
Using Lemma \ref{exponentialupperbound} we can uniformly upper bound
the measure of every halfspace around the median by
\[
  \wt\mu\left( \wt{H}(m_v v-tv,v)\right) \leq c_1e^{-c_2t}~,
\]
for $t\geq0$. Now using \eqref{medianlocalisation} and
\eqref{eq:xlocalisation}, we may uniformly bound the measure of
halfspaces around $x$. In particular, there exist constants
$c_{\gamma},c'_{\gamma}>0$ such that for all $t \geq 0$ and $u\in S^1$,
\begin{equation}
\label{half space weight up}
\wt\mu\left( \wt{H}\left( x- t u,u \right) \right)  \leq c_{\gamma} e^{-c_{\gamma}'t}~. 
\end{equation}
Next we use the fact that an isotropic log-concave
density in $\R^2$ is upper bounded by a universal constant. 
Obtaining upper bounds for log-concave densities is an important
problem in high-dimensional geometry. In particular, the so-called
\emph{isotropic constant} of a log-concave density $f$ is defined by
\[
  L_f:= \sqrt{\sup_{x\in\R^d} f(x)}  \sqrt[4]{\text{det}\left(
      \text{Cov}(X) \right)} ~,
\]
where $X$ is a random variable with densitiy $f$. It has a deep connection to Bourgain's ``slicing problem'' and the 
Kannan-Lovász-Simonovits 
conjecture, see, e.g., \cite{LUTWAK1993151,klartag2022bourgain}. Here we only need the simple fact that in a fixed dimension ($d=2$ in
our case) one has
$\sup_f L_f \leq K$
for a constant $K$, where the supremum is taken over all possible log-concave density functions. For an isotropic log-concave density, $L_f=\sqrt{\sup f}$, so indeed there exists an universal constant $K$ which upper bounds any log-concave isotropic density in dimension $2$.

Now we are ready to derive upper bounds for the right-hand side of \eqref{eq:conemeasure}.
To this end, we decompose the cone $C(x,u^{\perp},v^{\perp})$ into two
parts. For any $t>0$ we may write
\[
 \wt{\mu} \left(C(x,u^{\perp},v^{\perp})\right)
 \le \wt{\mu} \left(C(x,u^{\perp},v^{\perp}) \cap B(x,t)\right)
 + \wt{\mu} \left(C(x,u^{\perp},v^{\perp}) \cap \wt{H}(x-t u,u) \right)~,
\]
where $B(x,t)$ denotes the closed ball of radius $t$ centered at $x$.
Thus, from \eqref{half space weight up} and the upper bound on the
density, we obtain
\[
  \widetilde{\mu}\left( C(x,u^{\perp},v^{\perp})\right)  \leq
  K \pi t^2  \theta + c_{\gamma} e^{-c_{\gamma}'t}~,
\]
where $\theta \in [0,\pi]$ denotes the angle formed by vectors $u$ and $v$. We use a similar argument to get the saqme control on $C(x,-u^{\perp},-v^{\perp})$.
Choosing \linebreak $t=\log(1/\theta)/c_{\gamma}'$, \eqref{eq:conemeasure}
implies
\begin{equation*}
|\mu(H(x,u))-\mu(H(x,v))| \leq  C'_{\gamma}  \theta\log^2\left(\frac{1}{\theta}\right)
\end{equation*}
for a constant $C'_{\gamma}$ depending only on $\gamma$.
Since $\theta \leq \frac{\pi}{2}\|u-v\|$, we conclude that there exists a positive constant $C_{\gamma}$ such that
\begin{equation}\label{eq:Regularity}
|\mu(H(x,u))-\mu(H(x,v))| \leq  C_{\gamma} \|u-v\|\log^2\left(\frac{1}{\|u-v\|}\right)~.
\end{equation}
Now we are prepared to use Lévy's isoperimetric inequality. Choosing \linebreak
$A=\left\{v\in S^{d-1}: \ \mu(H(x,v))\geq 1/2 \right\}$, we clearly
have $\PROB\{U\in A\}= 1/2$ and therefore by \eqref{eq:LevyTheorem} 
\[
  \PROB\left\{ \inf_{v\in A} \|U-v\|\geq \epsilon \right\} \leq
  2e^{-(d-1)\epsilon^2/2}~.
\]
But for any $u\in S^{d-1}$ such that $\inf_{v\in A} \|u-v\| < \epsilon$,
\eqref{eq:Regularity} implies that
\[
  \mu(H(x,u)) >
  \frac{1}{2}-C_{\gamma}\epsilon\log^2\left(\frac{1}{\epsilon}\right)~,
\]
so
\[
  \PROB\left\{\mu(H(x,U))\leq
    \frac{1}{2}-C_{\gamma}\epsilon\log^2\left(\frac{1}{\epsilon}\right)
  \right\} \leq 2e^{-(d-1)\epsilon^2/2}~.
\]
Since $\ol{D}_k(x)=\min_{i=1,\dots, k}\mu(H(x,U_i))$ for $U_1,\ldots,U_k$ independently sampled uniformly on $S^{d-1}$, the union bound yields
\[
  \PP\left\{ \ol{D}_k(x) \leq
    \frac{1}{2}-C_{\gamma}\epsilon\log^2\left(
        \frac{1}{\epsilon}\right) \right\} \leq
  2ke^{-(d-1)\epsilon^2/2}~,
\]
concluding the proof.
\end{proof}

\section{Detection and localization of Tukey's median}
\label{sec:deep}

As explained in the introduction, a measure $\mu$ is called halfspace
symmetric if there exists a point $m\in \R^d$ with $\ol{d}(m)=1/2$.
Such a point is necessarily unique if $\mu$ has a density. Since it maximises Tukey's depth it is the Tukey median.
Clearly, for all $k\ge 1$, the random Tukey depth of the Tukey median
equals $\ol{D}_k(m)=1/2$. For $k>d$, if $\ol{D}_k(x)=1/2$ then $x$ is almost surely the Tukey median
and therefore, $1/2$ is trivially an exact estimate of the Tukey depth of
$m$. Here we show that, for any positive $\gamma$ bounded by some
constant,
already for values of $k$ that are of the order
of $d\log d$, all points that are at least a distance of 
order $\gamma \sqrt{d}$ away from $m$ have a random Tukey depth less than
$1/2-\gamma$, with high probability. This result implies that 
the Tukey median of isotropic log-concave, halfspace symmetric
distributions are efficiently estimated by the random Tukey median,
as stated in Corollary \ref{cor:deep}.

\begin{theorem}
\label{thm:deep}
Assume that $\mu$ is an isotropic log-concave, halfspace symmetric measure on $\R^d$ and let $m$ be the Tukey median.
Let $\delta>0$ and let $\gamma,r>0$ be such that
$r\ge 8\sqrt{2} e^4 \gamma$ and $r\le \min\left(e^{-4}/3, 8e^4\gamma \sqrt{d/2}\right)$.
There exists a universal constant $C>0$ such that, if
\[
  k \geq
  C\left(d\log \frac{r}{\gamma} +\log(1/\delta)\right) \frac{\gamma\sqrt{d}}{r}e^{C\gamma^2d/r^2}  ~,
\]
then
\[
  \PP\left\{ \sup_{x\in \R^d: \left\| x-m\right\| \geq r} \ol{D}_k (x)
    \geq \frac{1}{2}- \gamma  \right\} \ \leq \ \delta~. 
\]
In particular, by taking $r=8e^4\gamma\sqrt{d/2}$,
there exist universal constants $c,C>0$ such that for all $\gamma \le
c$, if
\[
  k \geq
  C\left(d\log d+\log(1/\delta)\right)~,
\]
then
\[
  \PP\left\{ \sup_{x\in \R^d: \left\| x-m\right\| \geq C\gamma\sqrt{d}}\ol{D}_k (x) \geq \frac{1}{2}- \gamma  \right\} \ \leq \ \delta~. 
\]
\end{theorem}

\begin{proof}
Without loss of generality, we may assume that $m=0$, that is, $\ol{d}(0)=1/2$.

The outline of the proof is as follows. First, we show that for a
fixed $x\in \R^d$ of norm $r$,
we have $\ol{D}_k(x)\leq \frac{1}{2}-2\gamma$ with high
probability.

Then we use an $\epsilon$-net argument to extend the control to the
sphere $r\cdot S^{d-1}$. To this end, we need to establish certain
regularity of the function $x\mapsto \ol{D}_k(x)$. 
We then use a monotonicity argument to extend the control to all points outside of the ball of radius $r$.

Recall that $f$ denotes the density of the measure $\mu$ and the random
vector $X$ has distribution $\mu$.
For any direction $u\in S^{d-1}$, we denote by
$\Phi_u(t)=\PROB\{\langle X,u\rangle \le t\}$ the cumulative
distribution function of the projection of $X$ in direction $u$.

Fix $x\in r\cdot S^{d-1}$. Since $\ol{D}_k(x)=\min_{i=1,\dots, k} \Phi_{U_i}(\langle x,U_i \rangle )$ ,
\begin{equation}
\label{proba depth product}
  \PP\left\{ \ol{D}_k(x) \geq \frac{1}{2}-2\gamma \right\}
 = 
\PP\left\{ \Phi_{U}(\langle x,U \rangle) \geq \frac{1}{2}-2\gamma   \right\}^k~.
\end{equation}
Next we bound the probability on the right-hand side. 
Since $\ol{d}(0)=1/2$, for all $u\in S^{d-1}$, $\Phi_u(0) =
1/2$. Clearly, the function $t\mapsto \Phi_u(t)$ is
non-decreasing, as it is a cumulative distribution function. Since
projections of an isotropic log-concave measure are also log-concave
and isotropic (see \cite[Section 3]{saumard2014log} and
\cite{Pre73}). Lemma \ref{density lower bound bis} in the Appendix
implies that for all $t \in [-e^{-4}/3, e^{-4}/3]$,
\[
  \Phi_{u}'(t)\geq e^{-4}/4~,
\]  
and therefore, for all such $t$, we have
\[
  \left\lvert \Phi_u(t)-\frac{1}{2} \right\rvert \ \geq \
  \frac{e^{-4}}{4}|t|~.
\]  
Since  $\left\| x\right\| = r \le e^{-4}/3$, we have $\lvert\langle
x,U_i\rangle\rvert \leq e^{-4}/3$ and hence
\begin{eqnarray*}
  \PP\left\{ \Phi_{U}(\langle x,U \rangle \geq
  \frac{1}{2}-2\gamma \right\}
& \leq & \PP\left\{
    \frac{1}{4e^4}\langle x,U \rangle \geq -2\gamma
  \right\}  \\ 
& = &
1- \PP\left\{ \langle \frac{1}{r} x,U \rangle \geq
   \frac{8e^4\gamma}{r} \   \right\}~.
\end{eqnarray*}
Since $\|\frac{1}{r} x\| =1$, the probability on the right-hand side
corresponds to the (normalized) measure of a spherical cap of height
$h=8e^4\gamma/r$.
Thus, we may further bound the expression on the right-hand side
by applying a lower bound for the measure of a spherical cap. 
In \cite[Lemma 2.1b]{BrGrKaKlLoSi01}, such a lower bound is provided for $ \sqrt{2/d}\leq h \leq 1$, which is
guaranteed by our
condition on $r$. We obtain
\[
  \PP\left\{ \Phi_{U}(\langle 	x,U \rangle \geq \frac{1}{2}
    - 2\gamma  \right\} \leq 1 -
  \frac{1}{6h\sqrt{d}}(1-h^2)^{\frac{d-1}{2}}   ~.
\]  
Hence, by \eqref{proba depth product} we have that for any $x$ with
$\|x\|=r \in [8\sqrt{2}e^4\gamma, e^{-4}/3]$,
\begin{eqnarray}
\label{local control}
  \lefteqn{
  \PP\left\{ \ol{D}_k(x)\geq \frac{1}{2} - 2\gamma \right\}
   \leq  \left(  1 -
           \frac{1}{6h\sqrt{d}}(1-h^2)^{\frac{d-1}{2}}\right)^k }
\nonumber \\
    & \le & \left(  1 -
            \frac{1}{6h\sqrt{d}} e^{-h^2(d-1)}\right)^k
\quad \text{(since $1-x \ge e^{-2x}$ for $x\in (0,1/2)$)}
\nonumber \\
    & \le & \exp\left(-
            \frac{k}{6h\sqrt{d}} e^{-h^2(d-1)}\right)
\quad \text{(since $1-x \le e^{-x}$ for $x\ge 0$).}
\end{eqnarray}
It remains to extend this inequality for a fixed $x$ to a uniform
control over all $\|x\|\ge r$. To this end, we need to establish regularity of the
function $x\mapsto \ol{D}_k(x)$.

Since $\left\|u\right\|=1$, the mapping
$x\mapsto\langle x,u\rangle$ is $1$-Lipschitz. $\Phi_u$ is the
cumulative distribution function of an isotropic, one-dimensional,
log-concave measure, and therefore its derivative is a log-concave
density with variance $1$. As stipulated
in Lemma \ref{density upper bound} in the Appendix, such a density is upper bounded by
$e^{7.1}$. Hence, for any $u\in S^{d-1}$,
$x\mapsto \Phi_u(\langle x, u\rangle)$ is $e^{7.1}$-Lipschitz. Furthermore,
since the minimum of Lipschitz functions is Lipschitz,
$x\mapsto \ol{D}_k(x)$ is also $e^{7.1}$-Lipschitz.

For $\epsilon>0$, an $\epsilon$-net of the sphere $r\cdot S^{d-1}$ is
a subset $N$ of $r\cdot S^{d-1}$
of minimal size such that for all $x\in r\cdot S^{d-1}$ there exists $y \in
N$ with $\|x-y\| \le \epsilon$. It is well known (see, e.g.,
\cite{Mat02}) that for all $\epsilon>0$, $r\cdot S^{d-1}$ has an
$\epsilon$-net $N_{\epsilon}$ of size at most
$|N_{\epsilon}|\le \left(\frac{2r}{\epsilon}+1\right)^d$.
Using the fact that 
$\ol{D}_k(x)$ is $e^{7.1}$-Lipschitz, by taking
$ \epsilon=e^{-7.1}\gamma$,  using \eqref{local control} and the union
bound, we have 

\begin{equation}
  \label{eq:epsilon net result}
 \PP\left\{ \sup_{x\in \R^d: \left\| x\right\| = r} \ol{D}_k(x)  \geq 
   \frac{1}{2}- \gamma \right\}  \leq 
 \left(\frac{2re^{7.1}}{\gamma}+1\right)^d \exp\left(-
            \frac{k}{6h\sqrt{d}} e^{-h^2(d-1)}\right)~.
\end{equation}
It remains to extend the inequality to include all points outside
$r\cdot S^{d-1}$. To this end, it suffices to show that for any $a\geq1$, 
$$\ol{D}_k(a x)\leq \ol{D}_k(x)~.$$
 
To see this, note that the deepest point $0$ has depth $1/2$, so every
closed half-space with $0$ on its boundary has measure
$1/2$. Hence, $\mu(H(x,u))<1/2$ if and
only if $0\notin H(x,u)$, which is equivalent to
$ \langle x,u \rangle < 0 $. On the event

$$\left\{\sup_{x\in\R^d: \left\| x\right\| = r} \ \ol{D}_k(x) <
  \frac{1}{2}-\gamma\right\}~,$$
for every $x\in r\cdot S^{d-1}$ there
exists an $i\in[k]$ such that $\mu(H(x,U_i))<1/2$. This implies that for
such an $i$, $\langle x,U_{i}\rangle < 0$, so for any $a\geq 1$, we have
$\langle a x,U_{i}\rangle \leq \langle x,U_{i} \rangle$. Since
$ \mu(H(x,U_{i}))=\Phi_{U_{i}}(\langle x,U_{i} \rangle)$ and that
$ \Phi_{U_{i}}$ is non decreasing, we have
\[
  \mu(H(a x,U_{i}))\leq \mu(H(x,U_{i}))~, 
\]
leading to $\ol{D}_k(a x) \leq \ol{D}_k(x)$ as desired. This extends
\eqref{eq:epsilon net result} to the inequality
\[
 \PP\left\{ \sup_{x\in \R^d: \left\| x\right\| \ge r} \ol{D}_k(x)  \geq 
   \frac{1}{2}- \gamma \right\}  \leq 
 \left(\frac{2re^{7.1}}{\gamma}+1\right)^d \exp\left(-
            \frac{k}{6h\sqrt{d}} e^{-h^2(d-1)}\right)~.
        \]
Recalling that $h=8e^4\gamma/r$ and that $r$ is bounded, this implies the announced statement.
\end{proof}

\section{Appendix}

In this section, we compile several properties of one-dimensional,
isotropic, log-concave densities. For a survey on log-concave
densities, see \cite{Sam18}. Before stating our results, we briefly explain why we only prove results for continuous log-concave densities and why the results directly extend to non-continuous log-concave probability densities. If a function $g$ is convex on $\R$, then there exists $a,b\in \R\cup \{-\infty,+\infty\} $ such that $g$ is continuous on $(a,b)$ and $g=+\infty$ on $(-\infty,a)\cup(b,+\infty)$. But, for any $\epsilon>0$, there exists a continuous convex function $\bar{g}$, which coincides with $g$ on $(a+\epsilon,b-\epsilon)$, and that is such that  the measure of density $e^{-\bar{g}(t)}$ has total mass and variance in $(1-\epsilon,1+\epsilon)$. Using this fact alone is enough to directly adapt the following proofs to any log-concave isotropic probability density.

\subsection{Lower bounds for log-concave densities}

\begin{lemma}
\label{density lower bound}
Let $f(t) = e^{-g(t)}$ be a log-concave probability density on $\R$
having variance $1$ and let $m$ denote its (unique) median. Then 
$$ e^{-g(m)}\geq \frac{e^{-4}}{2}~.$$
\end{lemma}

\begin{proof}
Without loss of generality, we may assume that $m=0$ and $g$ takes
its minimum on $\R^-$. Following the remark at the top of the appendix, we also assume that $g$ is continuous.

If $g(0)\leq 0$ the result is obvious, so suppose $g(0)>0$.
Since $g$ is convex, by taking its minimum on $\R^-$ it is non
decreasing on $\R^+$.
By continuity, there exists $L>0$ such that
$g(L)=2g(0)$.

From the convexity of $g$ we have that $g'(L)\geq \frac{g(0)}{L}$, and
therefore  for all $t \geq L$, 
\begin{equation}\label{ineq 5}
g(t)\geq g(L)+\frac{g(0)}{L}(t-L)\geq \frac{g(0)}{L}t~.
\end{equation}
Since $\int_{\R} f(x) dx =1$ and $0$ is the median,
\[
  \frac{1}{2}=\int_0^L e^{-g(t)}dt+\int_L^{\infty} e^{-g(t)}dt~.
\]
Using \eqref{ineq 5},
\[
  \int_L^{\infty} e^{-g(t)}dt \leq \int_L^{\infty}
  e^{-g(0)t/L}dt=\frac{L}{g(0)}e^{-g(0)}~.
\]
Because $g$ is non-decreasing on $\R^+$, 
\[
  \int_0^L e^{-g(t)}dt \leq e^{-g(0)}L~,
\]  
leading to
\begin{equation}\label{ineq 6}
\frac{1}{2} \leq \frac{L}{g(0)}e^{-g(0)}+e^{-g(0)}L=e^{-g(0)}L\left(1+\frac{1}{g(0)}\right)~.
\end{equation}
Now we use the fact that the variance equals $1$, that is,
\[
  1=\int_{-\infty}^{+\infty}t^2e^{-g(t)}dt -
  \left(\int_{-\infty}^{\infty}t e^{-g(t)}dt\right)^2~.
\]
Since the difference between the expectation and the median of any
distribution is at most the standard deviation, we have
$|\int_{-\infty}^{\infty}t e^{-g(t)}dt|\le 1$.
Moreover,  since $g$ is increasing on $\R^+$, for all $t \in [0,L]$ we
have $g(t)\leq 2g(0)$, and therefore $1 \geq
\int_0^{\infty}t^2e^{-g(t)}dt -1$
implies
\begin{equation}\label{ineq 7}
2 \geq \int_0^L t^2 e^{-2g(0)}dt = \frac{L^3}{3}e^{-2g(0)}~.
\end{equation}
From \eqref{ineq 6} we have
\[
  e^{-2g(0)}L^3  e^{-g(0)}\left(1+\frac{1}{g(0)}\right)^3 \geq
  \frac{1}{8}~.
\]
Hence, by plugging the inequality into \eqref{ineq 7}, we get
\begin{equation}\label{ineq 8}
e^{-g(0)}\left(1+\frac{1}{g(0)}\right)^3 \geq \frac{1}{48}~.
\end{equation}
Note that the function $h:t \mapsto e^{-t}\left(1+\frac{1}{t}\right)^3$
is non increasing on $\R^+ $.
To conclude, observe that
\begin{itemize}
\item if $g(0)\leq 4.5$, then $e^{-g(0)}\geq \frac{e^{-4}}{2}$.
\item if $g(0)>4.5$, then
$$ h(g(0)) < \frac{1}{48},$$
contradicting \eqref{ineq 8}.
\end{itemize}
\end{proof}

The next result shows that an isotropic log-concave density
is in fact bounded from below by a universal constant on an interval
around the median.

\begin{lemma}
\label{density lower bound bis}
Let $f(t) = e^{-g(t)}$ be a log-concave probability density on $\R$
having variance $1$ and median $m=0$. Then 
for all $t\in \left[-\frac{1}{3e^4},\frac{1}{3e^4}\right]$,
\[
  f(t) \geq \frac{1}{4e^4}~.
\]  
\end{lemma}

\begin{proof}
Denote $\alpha=1/(3e^4)$ and suppose that there exists $t\in\left[
  -\alpha,\alpha\right]$
such that $f(t)< 1/(4e^4)$.
Since log-concave densities are  unimodal, on $\left[
  -\alpha,\alpha\right]$
the density $f$ reaches its minimum on an endpoint of the
interval.
Without any loss of generality, assume that
$$ e^{-g(\alpha)} < \frac{1}{4e^4}~,$$
that is,
$$g(\alpha) > 4+\log(4)~.$$
By the convexity of $g$,
for all $t \geq \alpha$,
\[
  g(t)\geq  \frac{g(\alpha)-g(0)}{\alpha}(t-\alpha)+g(\alpha)~.
\]
Since by Lemma \ref{density lower bound},  $g(0)\leq4+\log(2)$, we get  that
for all $t \geq \alpha$
\[
  g(t)\geq \frac{\log(2)}{\alpha}(t-\alpha)+\log(4e^4)~.
\]
It follows that
$$
\int_{\alpha}^\infty e^{-g(t)}dt \leq \frac{1}{4e^4}  \cdot \frac{\alpha}{\log(2)}~. $$
We also prove in Lemma \ref{density upper bound} below that $\sup_{t\in\R}e^{-g(t)}\leq e^4$, so
$$ \int_0^\alpha e^{-g(t)}dt \leq \alpha e^4~. $$
Using the fact that $0$ is the median, we get
\[
  1=\frac{1}{2}+\int_{\R^+}e^{-g(t)}dt \leq \frac{1}{2}+ \alpha\left(
    e^4+\frac{1}{4e^4} \frac{1}{\log(2)} \right)~.
\]  
But
$$ \alpha\left( e^4+\frac{1}{4e^4} \frac{1}{\log(2)} \right)<\frac{1}{2}~, $$
which is a contradiction. This concludes the proof.
\end{proof}

\subsection{Upper bounds for log-concave densities}

\begin{lemma}
  \label{density upper bound}
Let $f(t) = e^{-g(t)}$ be a log-concave probability density on $\R$
having variance $1$. Then
$$ \sup_{t\in\R}e^{-g(t)}\leq e^{7.1}. $$
\end{lemma}

\begin{proof}
  Without loss of generality, we may assume that
  $g(0)=\inf_{t\in\R}g(t)$ and
  $\int_0^{\infty}t^2e^{-g(t)}dt \geq 1/2$. Following the remark at the top of the appendix, we also assume that $g$ is continuous.
  
  First note that if $g(0)\geq0$, then there's nothing to prove,
  so suppose that $g(0)<0$.
By the intermediate value theorem there exists $L>0$ such that $g(L/2)=g(0)/2$.
Since $g$ is convex and $\int_{\R} \exp(-g(t))dt=1$, we have
\begin{equation}\label{ineq total mass}
\frac{L}{2}
 e^{-g(0)/2}\leq 1~.
\end{equation}
Since $g$ has a non-decreasing derivative, for all $t \geq L/2$,
\[
  g'(t)\geq -\frac{g(0)}{2}\cdot \frac{2}{L}=-\frac{g(0)}{L}~.
\]
Then for all $t\geq L/2$,   $g(t)\geq g(L/2)-\frac{g(0)}{L}(t-L/2) \geq  g(0)-\frac{g(0)}{L}(t-L)$, which implies
\[
\int_{L/2}^{\infty}t^2e^{-g(t)}dt  \leq  e^{-2g(0)} \int_{L/2}^{\infty}t^2e^{\frac{g(0)}{L} t}dt~.
\]
Since for $c>0$
$$ \int_{L/2}^{\infty}t^2e^{-ct}dt=\left( \frac{L^2}{4c}+\frac{L}{c^2}+\frac{2}{c^3}\right)  e^{-cL/2}, $$
taking $c=-g(0)/L$, which is positive,
\begin{equation}\label{ineq 1}
\int_{L/2}^{\infty}t^2e^{-g(t)}dt \ \leq\left( \frac{-L^3}{4g(0)}+\frac{L^3}{g(0)^2}-\frac{2L^3}{g(0)^3}\right) e^{-3g(0)/2}~.
\end{equation}
Next we establish a lower bound for
$\int_{L/2}^{\infty}t^2e^{-g(t)}dt$.
The fact that the second moment on $\R^+$ is greater than $1/2$ implies 
$$ \int_{L/2}^{\infty}t^2e^{-g(t)}dt \geq  \frac{1}{2}  - \int_{0}^{L/2}t^2e^{-g(t)}dt~.$$
It is immediate from the fact that $g$ reaches its minimum in $0$ that
$$  \int_{0}^{L/2}t^2e^{-g(t)}dt \leq \frac{L^3}{24}e^{-g(0)}~, $$
leading to
\begin{equation}\label{ineq 2}
\int_{L/2}^{\infty}t^2e^{-g(t)}dt  \geq \frac{1}{2}- \frac{L^3}{24}e^{-g(0)}~.
\end{equation}
Comparing \eqref{ineq 1} and \eqref{ineq 2}, we obtain
\begin{equation*}
\frac{1}{2}- \frac{L^3}{24}e^{-g(0)}\leq L^3 \left( \frac{-1}{4g(0)}+\frac{1}{g(0)^2}-\frac{2}{g(0)^3}\right) e^{-3g(0)/2}~,
\end{equation*}
leading to
\begin{equation}\label{ineq 3}
\frac{1}{2}\leq L^3 \left(\frac{e^{g(0)/2}}{24} - \frac{1}{4g(0)}+\frac{1}{g(0)^2}-\frac{2}{g(0)^3}\right) e^{-3g(0)/2}~. 
\end{equation}
From \eqref{ineq total mass} we have $L^3e^{-3g(0)/2}\leq 8$, which, plugged into \eqref{ineq 3} yields
\begin{equation*}
1 \leq 16 \left(\frac{e^{g(0)/2}}{24} - \frac{1}{4g(0)}+\frac{1}{g(0)^2}-\frac{2}{g(0)^3}\right)~.
\end{equation*}
And so
\begin{equation}\label{ineq 4}
1\leq \frac{2}{3}e^{g(0)/2} - \frac{4}{g(0)}+\frac{16}{g(0)^2}-\frac{32}{g(0)^3}~.
\end{equation}
The function $h: t\mapsto
\frac{2}{3}e^{t/2} - \frac{4}{t}+\frac{16}{t^2}-\frac{32}{t^3}$ is
non-decreasing on $\R^-$. To conclude the proof, note that
if $g(0) \geq -7.1$, then $e^{-g(0)}\leq e^{7.1}$. Otherwise,
if $g(0)< -7.1$, then, since $h$ is non-decreasing, 
$$h(g(0))\leq h(-7.1)\leq0.99<1~,$$
which contradicts \eqref{ineq 4}.
\end{proof}

It is known (see, e.g., \cite{10.1214/09-EJS505}) that for any
log-concave density $f$ on $\R^d$, there exist positive constants
$\alpha,\beta$ such that $f(x)\leq e^{-\alpha \|x\|+\beta}$ for all
$x\in \R^d$. The next lemma shows that for isotropic log-concave densities on
$\R$ with
median at $0$, one may choose $\alpha$ and $\beta$ independently of $f$.

\begin{lemma}
\label{exponentialupperbound}
Let $f(t) = e^{-g(t)}$ be a log-concave probability density on $\R$
having variance $1$ and median $m=0$.
Then there exist universal constants $\alpha,\beta>0$ such that
for all $t \in \R$,
\[
  f(t) \leq \alpha e^{-\beta|t|}~.
\]
\end{lemma}

\begin{proof}
By Lemma \ref{density lower bound} we have $e^{-g(0)}\geq e^{-4}/2$. The log-concavity of the density implies that on any given interval, the minimum is reached at one of the endpoints of the interval. Thus,
$$ \int_0^{2e^4}e^{-g(t)}dt \geq 2e^4  \min\left(e^{-g(2e^4)},e^{-4}/2\right)~. $$
Since $0$ is the median of $f$, $2e^4 \min(e^{-g(2e^4)},e^{-4}/2)\leq 1/2$. Thus,
\begin{equation}\label{eq: exponential upper bound}
e^{-g(2e^4)} \leq \frac{e^{-4}}{4}~.
\end{equation}
A mirror argument proves that $ e^{-g(-2e^4)} \leq \frac{e^{-4}}{4}$.
By Lemma \ref{density lower bound}, $g(0)\leq \log(2)+4$ and
\eqref{eq: exponential upper bound} implies $g(2e^4)\geq
\log(4)+4$. Using the convexity of $g$ yields that for all $t \geq
2e^4$,
\[ g(t)\geq 4+\log(4)+(t-2e^4) \frac{\log(2)}{2e^4}~,
\]
so, using Lemma \ref{density upper bound} which states that $g(0)\geq -7.1$,
for all $t\in \R^+$,
\[  g(t)\geq  -7.1+(t-2e^4) \frac{\log(2)}{2e^4}  ~.
\]
A identical argument on $\R^-$ concludes the proof of the Lemma.
\end{proof}

\subsection{Proof of Lemma \ref{TheEmpGap}}

\begin{proof}
To prove the first inequality, observe that
\begin{eqnarray*}
\sup_{x\in \R^d} |\ol{d}(x) - d_n(x)|
& = &
\sup_{x\in \R^d} \left|\inf_{u\in S^{d-1}}\mu(H(x,u)) 
- \inf_{u\in S^{d-1}}\frac{1}{n}\sum_{i=1}^n \IND_{X_i \in H(x,u)} \right|
\\
& \le &
\sup_{x\in \R^d} \sup_{u\in S^{d-1}}\left|\mu(H(x,u))
- \frac{1}{n}\sum_{i=1}^n \IND_{X_i \in H(x,u)} \right|~.
\end{eqnarray*}
The first inequality of the Lemma follows from the 
Vapnik-Chervonenkis inequality (see, e.g., \cite[Theorem 12.5]{devroye2013probabilistic}) and the fact that the 
{\sc vc} dimension of the class of all halfspaces $H(x,u)$
equals $d+1$.

\noindent The second inequality is proved similarly, combining it with
a simple union bound that gives a better bound when $\log(k) \ll d$.
\end{proof}

\subsection{Proof of Lemma \ref{lem:PhiControl}}

\begin{proof}
Because $\Phi$ is convex, increasing on $\R^-$, and that $\Phi(\R^-)=[0,1/2]$, then $\Phi^{-1}$ is concave on $[0,1/2]$. Thus
$$  \frac{\Phi^{-1}(\epsilon/2)-\Phi^{-1}(\epsilon/4)}{\epsilon/4}  \geq (\Phi^{-1})'(\epsilon/2)~. $$
Using the fact that $(\Phi^{-1})'=1/(\Phi'(\Phi^{-1}))$ and $\Phi'(t)=\frac{1}{\sqrt{2\pi}}e^{-t^2/2}$,
\begin{equation}\label{eq:EpsilonDifference}
\Phi^{-1}(\epsilon/2)-\Phi^{-1}(\epsilon/4)\geq \frac{\epsilon}{4}\sqrt{2\pi}e^{\Phi^{-1}(\epsilon/2)^2/2 }~.
\end{equation}
By Gordon's inequality for the Mills' ratio \cite{Gor41}, for $t\leq0$,
$$ \Phi(t)\geq -\frac{1}{\sqrt{2\pi}}\frac{t}{t^2+1}e^{-t^2/2}~, $$
and therefore
$$ t \geq \Phi^{-1}\left(  -\frac{1}{\sqrt{2\pi}}\frac{t}{t^2+1}e^{-t^2/2} \right)~,$$
leading, for $t<-1$, to
\begin{equation}\label{eq:PhiInverseUpperBound}
t\geq  \Phi^{-1}\left(  -\frac{e^{-t^2/2}}{10t} \right)~.
\end{equation}
Choosing $t_{\epsilon}=-\sqrt{2
  \log(1/\epsilon)}\sqrt{1-\frac{\log\log(1/\epsilon)}{\log(1/\epsilon)}}$
for $\epsilon<e^{-2}$ and noting that
 $$  -\frac{e^{-t_{\epsilon}^2/2}}{10t_{\epsilon}}\geq \epsilon/2, $$
\eqref{eq:PhiInverseUpperBound} implies that
$$-\sqrt{2 \log(1/\epsilon)}\sqrt{1-\frac{\log\log(1/\epsilon)}{\log(1/\epsilon)}}\geq \Phi^{-1}(\epsilon/2)~.$$
Plugging this inequality into \eqref{eq:EpsilonDifference}
$$ \Phi^{-1}(\epsilon/2)-\Phi^{-1}(\epsilon/4)\geq
\frac{\sqrt{2\pi}}{4\log(1/\epsilon)}~, $$
proving Lemma \eqref{lem:PhiControl}.

\end{proof}

\section*{Acknowledgements}

We would like to thank Imre B\'ar\'any, Shahar Mendelson, Arshak Minasyan, Bill Steiger, and Nikita Zhivotovsky for helpful discussions.
We also thank Reihaneh Malekian for her thorough reading of the original manuscript
and for pointing out some inaccuracies.

\bibliographystyle{siamplain}
\bibliography{bibli}

\end{document}